\documentclass[letterpaper, 10 pt, conference]{ieeeconf}  

\IEEEoverridecommandlockouts                              

\usepackage{graphics} 
\usepackage{epsfig} 
\usepackage{amsmath,amssymb}
\usepackage{mathtools}
\usepackage{xcolor}
\usepackage{colortbl}
\usepackage{algorithm}
\usepackage[noend]{algpseudocode}

\newcommand{\argmin}{\operatornamewithlimits{arg\,min}}

\algrenewcommand\algorithmicindent{1.5em}

\DeclarePairedDelimiter\abs{\lvert}{\rvert}%
\DeclarePairedDelimiter\norm{\lVert}{\rVert}%
\makeatletter
\let\oldabs\abs
\def\abs{\@ifstar{\oldabs}{\oldabs*}}
\let\oldnorm\norm
\def\norm{\@ifstar{\oldnorm}{\oldnorm*}}
\makeatother
\newcommand{\defeq}{\vcentcolon=}

\newtheorem{definition}{Definition}

\newtheorem{remark}{Remark}
\newtheorem{theorem}{Theorem}

\newtheorem{assumption}{Assumption}
\newtheorem{problem}{Problem}

\title{\LARGE \bf
Efficient and Safe Exploration in Deterministic\\ Markov Decision Processes with Unknown Transition Models
}

\author{Erdem B\i y\i k$^{*1}$, Jonathan Margoliash$^{*2,3}$, Shahrouz Ryan Alimo$^{2}$, and Dorsa Sadigh$^{1,4}$
\thanks{$^*$ First two authors contributed equally and are listed in alphabetical order.}
\thanks{Emails: {\tt\footnotesize ebiyik@stanford.edu}, {\tt\footnotesize jmargoli@eng.ucsd.edu}, {\tt\footnotesize sralimo@jpl.nasa.gov}, {\tt\footnotesize dorsa@cs.stanford.edu}}
\thanks{$^{1}$ Erdem B\i y\i k and Dorsa Sadigh are with Electrical Engineering, Stanford University, CA, 94305, United States.}%
\thanks{$^{2}$ Jonathan Margoliash and Shahrouz Ryan Alimo are with Jet Propulsion Laboratory, California Institute of Technology, CA, 91109, United States.}%
\thanks{$^{3}$ Jonathan Margoliash is with UC San Diego Jacobs School of Engineering, CA, 92093, United States.}%
\thanks{$^{4}$ Dorsa Sadigh is with Computer Science, Stanford University, CA, 94305, United States.}%
}

\begin{document}

\maketitle
\thispagestyle{empty}
\pagestyle{empty}

\begin{abstract}

We propose a safe exploration algorithm for deterministic Markov Decision Processes with unknown transition models. Our algorithm guarantees safety by leveraging Lipschitz-continuity to ensure that no unsafe states are visited during exploration. Unlike many other existing techniques, the provided safety guarantee is deterministic. Our algorithm is optimized to reduce the number of actions needed for exploring the safe space. We demonstrate the performance of our algorithm in comparison with baseline methods in simulation on navigation tasks.

\end{abstract}

\section{Introduction}
Guaranteeing safety is a vital issue for many modern robotics systems, such as unmanned aerial vehicles (UAVs), autonomous cars, or domestic robots~\cite{sadigh2016safe,dey2016fast,katz2017reluplex}. One approach is to attempt to specify all potential scenarios a robot may encounter a priori. However, this is usually impractical due to the fact that such solutions are either computationally expensive to compute or that robots today need to deal with uncertain and diverse environments.
Hence, we need to design algorithms for robots that can \emph{safely} and \emph{autonomously} learn about the uncertain environment they live in, which can potentially address both problems \cite{argall2009survey, kober2012reinforcement}.

Reinforcement learning algorithms autonomously perform exploration, and they have shown promising results in many fields of artificial intelligence. Therefore, it is natural to explore the implications of reinforcement learning in robotics. Unlike most applications in artificial intelligence, where unsafe outcomes from learning can occur in simulation, in robotics, we would need to avoid unsafe scenarios at all costs.
Therefore, \emph{safe learning}, which is the process of applying a learning algorithm such as reinforcement learning while still satisfying a set of safety specifications, has attracted great interest in recent years. Safety usually has two main interpretations: one is related to stochasticity of the environment, where the goal is to guarantee staying in a given performance bound as is commonly studied in robust control~\cite{coraluppi1999risk,heger1994consideration,sato2001td,borkar2002q,gaskett2003reinforcement,geibel2005risk, aswani2013provably}. The second interpretation is due to the system falling in an undesirable physical state, which is a common interpretation used in robotics~\cite{abbeel2005exploration,abbeel2010autonomous,berkenkamp2016bayesian,ames2017control,akametalu2014reachability}. In this paper, we focus on safe learning and exploration as the avoidance of `unsafe' - physically undesirable - states. We refer to \cite{garcia2015comprehensive} for a survey on safe reinforcement learning.


To perform exploration in safety-critical systems, prior knowledge of the task is often incorporated into the exploration process. 
 In \cite{turchetta2016safe} and \cite{wachi2018safe},  the authors proposed a method to safely explore a deterministic Markov Decision Process (MDP) using Gaussian processes. In their work, they assumed the transition model is known and that there exists a predefined safety function. Both of these assumptions can be quite restrictive when the system is going to operate in unknown environments. In our work, we plan to address both of these challenges by considering unknown transition models, and no access to a predefined safety function.
Similarly, other work has considered reachability analysis as well as Gaussian processes to perform safe reinforcement learning~\cite{akametalu2014reachability,gillula2011guaranteed}, and used a safety metric to improve their algorithm. However, it is not trivial to derive an appropriate safety metric in many robotics tasks.
Other techniques utilize teacher demonstrations to avoid unsafe states \cite{abbeel2005exploration, abbeel2010autonomous,garcia2012safe}. However, teaching demonstrations are usually difficult to capture as operating robots with high degrees of freedom can be challenging, especially if the system dynamics are unknown due to the existing uncertainty in the environment. 

In our work, we propose an algorithm to safely and autonomously explore a deterministic MDP whose transition function is \emph{unknown}. 
We take a natural definition of safety, similar to \cite{moldovan2012safe}: \emph{If we can recover from a state $s$, i.e. we can move from it to a state that is known to be safe, then the state $s$ is also safe.} Instead of relying on Gaussian processes or other estimation procedures, our exploration algorithm directly leverages the underlying continuity assumptions, and so guarantees safety deterministically. We demonstrate our algorithm in simulation on two different navigation tasks.

\section{Problem Definition}
\label{sec:problem_definition}

Our goal in this project is to design an algorithm for a robot to safely and efficiently explore uncertain parts of the environment. We want an algorithm that deterministically ensure safety and expand the size of the known safe state set. The theoretical work in this section will build towards formalizing this goal in equation (\ref{goal}). 

\subsection{Introductory Assumptions}
First, we begin by formalizing our interactions with the environment.

\begin{assumption}
  We model the dynamical system living in an unknown environment as a deterministic MDP \cite{sutton1998reinforcement}. Such an MDP is a tuple $\langle\mathcal{S}, \mathcal{A}, f(s,a)\rangle$ with a set of states $\mathcal{S}$, a set of actions $\mathcal{A}$, and an unknown deterministic transition model $f : \mathcal{S} \times \mathcal{A} \to \mathcal{S}$. Let $s_0 \in \mathcal{S}$ be the initial state of the MDP.
\end{assumption}

For example, for a quadrotor, the states in $\mathcal{S}$ might be the quadrotor's pitch, yaw and roll, the quadrotor's angular and linear velocities, and its height from the ground, all concatenated together to form a vector in $\mathbb{R}^{k_1}$. The actions in $\mathcal{A}$ could be the fixed rotation speeds of the rotors (a vector in $\mathbb{R}^{k_2}$), and the transition function $f$ would apply those speeds to the rotors over a fixed interval of duration $dt$.

The algorithm we will develop is applicable to finite state and action spaces. When $\mathcal{S}$ or $\mathcal{A}$ are fundamentally continuous, our algorithm can be applied to finite, fine-grained discretizations of those space. We show in Section~\ref{sec:simulations_and_results} how to handle this discretization.

We now make a few definitions that will help us reason about our knowledge of the environment and our ability to take a sequence of actions to efficiently explore the environment.

\begin{definition}
We denote knowledge the actor has about the transition function $f$ as a set $\mathcal{T} \subset \mathcal{S} \times \mathcal{A} \times \mathcal{S}$ such that the following implication relation holds:
\begin{align}
(s, a, s') \in \mathcal{T} \implies f(s, a) = s'
\end{align}
\end{definition}
\begin{definition}
Let $A$, with appropriate superscripts when necessary, denote a sequence of actions; $A_i$ the $i^{\text{th}}$ action in the list and $\abs{A}$ the cardinality. We overload the transition function $f$ such that $f(s,A)=s'$ if and only if taking the actions $A_1, A_2, \dots, A_{\abs{A}}$ sequentially from state $s$ yields the final state $s'$. Lastly, $\mathbb{A}$ denotes the set of all possible ordered sequences of actions $A$.
\end{definition}

Without further assumptions about $f$, it is not possible to perform safe exploration from $s_0$, because we cannot take any action from the initial state with limited knowledge of the action's safety. Therefore, we assume we are given an initial safe set $S_0 \subseteq \mathcal{S}$ such that the initial state of the system is in this set, i.e. $s_0 \in S_0$. Furthermore, we assume as in \cite{turchetta2016safe} that for any $s,s'\in S_0$, we are given a list of actions $A$ such that $f(s, A)=s'$. 

\begin{assumption} Formally, we assume we are given an initial safe set $S_0 \subseteq \mathcal{S}$ such that $s_0 \in S_0$, an initial knowledge set $\mathcal{T}_0$, and restate the above assumption as:
\begin{align*}
& \forall s, s' \in S_0, \exists k \in \mathbb{Z}_{\geq0}, \exists a_1,\dots,a_k\in\mathcal{A}, \exists s_1,\dots,s_{k-1}\in\mathcal{S} \\
& (s, a_1,s_1), (s_1,a_2,s_2), \dots, (s_{k\!-\!1},a_k,s') \in \mathcal{T}_0
\end{align*}
\end{assumption}

We further make a set of assumptions on Lipschitz continuity of the transition function to enable safe exploration.
\begin{assumption}
\label{assumption:lipschitz}
\begin{itemize}
  \item[]
  \item $f(s,a)$ is $L_s$-Lipschitz continuous over the states with some distance metric $d^s:\mathcal{S}\times\mathcal{S}\to\mathbb{R}$:
	\begin{align}
	d^s(f(s,a), f(s',a)) \leq L_s d^s(s,s')
	\end{align}
	\item Similarly, $f(s,a)$ is $L_a$-Lipschitz continuous over the actions with $d^s$ and the additional distance metric $d^a:\mathcal{A}\times\mathcal{A}\to\mathbb{R}$:
	\begin{align}
	d^s(f(s,a), f(s,a')) \leq L_a d^a(a,a')
	\end{align}
\end{itemize}
\end{assumption}
Practically, the Euclidean distance is often used for both $d^s$ and $d^a$.

Note that these requirements are mild and naturally satisfied in most domains: If we take the same action from two similar states, we will end up in similar states; and if we take similar actions from the same state, we will again end up in similar states. For this algorithm, we assume that $L_s$ and $L_a$ have been estimated via some prior methodology, and we note that larger than optimal values of these constants, while leading to less efficient algorithms, still satisfy Assumption \ref{assumption:lipschitz}.  

With this setup, we can now define our notion of safety.
\begin{definition}
\label{def:safety1}
  We define $s \in \mathcal{S}$ \textit{to be safe with respect to} $S_0$ and knowledge set $\mathcal{T}$ if there exists a recovery algorithm that can use the information in $\mathcal{T}$ and the Lipschitz assumptions to confidently produce actions which will transition the MDP from $s$ into $S_0$ after a finite number of steps. We define that action $a$ \textit{is safe at state} $s$ if all possible outcomes of $f(s, a)$, with respect to $\mathcal{T}$ and Lipschitz assumptions, are safe. When calling a set safe, the choice of $S_0$ should be clear from context.
\end{definition}

For instance, in our quadrotor example, $S_0$ might include various hovering states, and thus we consider any state to be safe if we can return to a hovering position from that state. We note that our notion of safety is similar to the safety definition of \cite{moldovan2012safe}. 

This definition of safety, while theoretically satisfying, is not computable in its current form, as we have not described what it means to ``use information to confidently produce actions''. With a bit more work, we will do so in Definition \ref{def:safety2} below.

We now state some important observations about our definition of safety.
\begin{remark}
  Suppose $S$ is a safe set with respect to some knowledge set $\mathcal{T}$. Then:
\begin{align}
  S \subseteq \bar S \defeq \{s \in \mathcal{S} | \exists A \in \mathbb{A} \text{ s.t. } f(s, A) \in S_0\}
\label{eq:safety_definition}
\end{align}
\end{remark}
However, this upper bound on the safe set cannot normally be calculated, as $f$ is unknown.

\begin{remark}
  \label{remark:new knowledge}
Suppose $S$ is safe with respect to some $\mathcal{T}$, and that $\mathcal{T} \subseteq \mathcal{T'}$. Then $S$ is safe with respect to $\mathcal{T'}$.
\end{remark}

\subsection{Computing the Safe Set}
We now utilize the knowledge set $\mathcal{T}$ to determine which states and actions are safe.

\begin{definition}
  In order to handle unknown states, we define an uncertain transition function $f_u$, parameterized by knowledge $\mathcal{T}$, that maps each state-action pair to all of its possible outcomes:
	\begin{align*}
	f_u(s,a;\mathcal{T}) \defeq 
	\bigcap\limits_{(s',a',s'')\in\mathcal{T}} \phi(s'', L_sd^s(s',s)\!+\!L_ad^a(a',a))
	\end{align*}
	where $\phi(s,x) \defeq \{s' \in \mathcal{S} | d^s(s,s') \leq x\}$ denotes the hypersphere centered at $s$ with radius $x$ over the distance metric $d^s$. A visualization of $f_u$-function is shown in Fig.~\ref{fig:lipschitz}.
\end{definition}
\begin{figure}[h]
	\centering
	\includegraphics[width=\columnwidth]{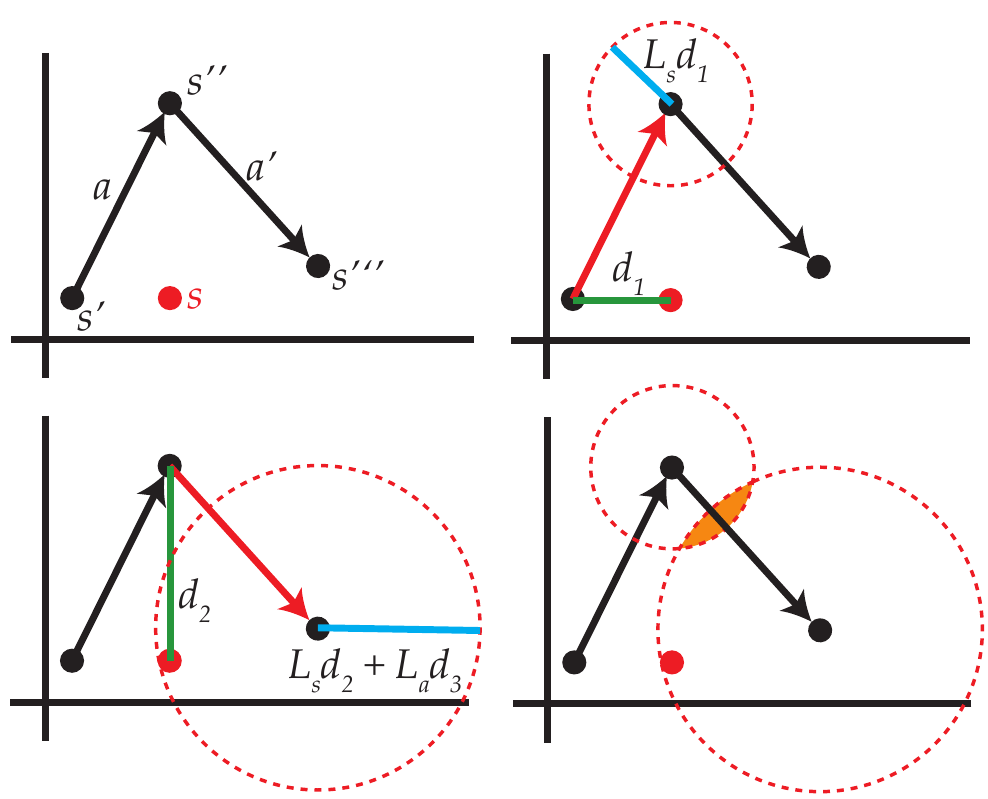}
	\vspace*{-20pt}
	\caption{A visualization of Lipschitz assumption and $f_u$-function has been shown. \textbf{Top left.} We are given two state transitions: $f(s',a)=s''$ and $f(s'',a')=s'''$. We want to know the possible outcomes when we take action $a$ from state $s$. \textbf{Top right.} Considering the first transition, the outcome of $f(s,a)$ must be within a circle of radius $L_sd_1$ from $s''$, where $d_1 = d_s(s',s)$. \textbf{Bottom left.} Similarly, when we consider the second transition, we note that the outcome must be within a circle of radius $L_sd_2 + L_ad_3$ from $s'''$, where $d_2 = d_s(s'',s)$ and $d_3=d_a(a,a')$. \textbf{Bottom right.} The states that satisfy both of those conditions lie on the intersection, which is shown as the shaded area. It is exactly equal to $f_u(s,a;\mathcal{T})$ where $\mathcal{T}=\{(s',a,s''),(s'',a',s''')\}$.}
	\label{fig:lipschitz}
	\vspace*{-10pt}
\end{figure}

\begin{definition}
  If $S$ is a safe set with respect to knowledge $\mathcal{T}$, then we define the expansion function as:
  \begin{align}
  E_\mathcal{T}(S) \defeq S \cup \{s \in \mathcal{S} | \exists a \in \mathcal{A} \text{ s.t. } f_u(s,a;\mathcal{T})\subseteq S\}
  \end{align}
  We also let $E_{\mathcal{T}}^{(k)}(S) = E_{\mathcal{T}}(E_{\mathcal{T}}^{(k-1)}(S))$ with $E_{\mathcal{T}}^{(1)}(S) = E_{\mathcal{T}}(S)$. Moreover, we denote $\bar E_{\mathcal{T}}(S) = \lim_{k\to\infty}E_{\mathcal{T}}^{(k)}(S)$ to be the fixed point of the expansion function.
\end{definition}

Using these definitions, we make the following crucial observation that is similar to Eq.~\eqref{eq:safety_definition}:
\begin{theorem}
  \label{expansionThm}
  If $S$ is safe with respect to $\mathcal{T}$, then so is $E_\mathcal{T}(S)$.
\end{theorem}
\begin{proof}
  Let $s \in E_\mathcal{T}(S)$. Then, at least one of the following is true:
  \begin{itemize}
  	\item $s \in S$
  	\item $\exists a \in A$ such that $f_u(s, a; \mathcal{T}) \subseteq S$
  \end{itemize}
If $s \in S$, then since $S$ is safe with respect to $\mathcal{T}$, so is $s$. Otherwise, since there exists an $a$ such that $f_u(s, a; \mathcal{T}) \subseteq S$, even if $f(s,a)$ is still unknown, we do know $f(s, a) \in S$, which is safe. From there we can use the recovery algorithm provided by Definition \ref{def:safety1} for $S$ to return to $S_0$.
\end{proof}

\begin{definition}
If we take action $a_i$ from the state $s_i$ at some time step $i\geq0$, then we recursively define our knowledge after taking step $i$ to be $\mathcal{T}_{i + 1} = \mathcal{T}_i \cup \{(s_i, a_i, f(s_i,a_i))\}$.
\end{definition}
\begin{definition}
\label{def:safety2}
  For timestep $i \geq 0$, recursively define the safe set after step $i$ to be
  \begin{align}
    S_{i+1} = \bar E_{\mathcal{T}_{i}}(S_{i})
    \label{eq:safe_set_expansion}
  \end{align}
\end{definition}

Using Theorem~\ref{expansionThm}, we see that this definition is justified, i.e. $S_{i+1}$ is safe with respect to $\mathcal{T}_i$ for all $i$. We now have a computable set which we can pair with our more theoretical definition of safety, Definition \ref{def:safety1}.
\begin{remark}
  At step $i$, a state $s$ is safe if and only if $\exists a \in \mathcal{A}$ such that the action $a$ is safe at $s$.
\end{remark}

We are now ready to state our goal:
\begin{problem}
  Our goal is to maximize the rate of safe exploration over the number of actions taken, given an initial safe set $S_0$ and state $s_0\in S_0$. 
Formally, 
\begin{align}
\label{goal}
  \max_{a_1,a_2,\dots,a_m} & \frac{\abs{S_m}}{m} \nonumber\\
  \text{subject to } & s_{k+1} = f(s_k,a_{k+1}) \in S_k\subseteq \bar S \subseteq \mathcal{S}, \nonumber\\
  & a_{k+1}\in\mathcal{A}, \text{ for } k=0,1,\dots,m-1
\end{align}
where $m$ is the total number of actions taken, and $S_i$ is the set of states deterministically known to be safe by the algorithm after $i$ actions, and $\bar{S}$ is as defined in Eq.~\eqref{eq:safety_definition}.
\end{problem}
\section{Algorithm}
\label{sec:algorithm}

In order to present our algorithm that efficiently expands the safe set, we first introduce some notations and functions.
\subsection{Preliminaries}

\begin{definition}
We define the \emph{path-knowledge} function $q: \mathcal{S}\times\mathcal{S}\to\{0,1\}$ parametrized with $\mathcal{T}$ that carries the transition triplets:
\begin{align*}
	q(s,s';\mathcal{T}) = \begin{cases}
	1 & \text{if } \exists A: (s,A_1,f(s,A_1)),\\&(f(s,A_1), A_2, f(f(s,A_1),A_2)), \dots \in \mathcal{T} \\
	&\land\:f(s,A) = s'\\
	0 & \text{otherwise}
	\end{cases}
\end{align*}
That is, $q(s,s';\mathcal{T})=1$ if and only if there exist triplets in $\mathcal{T}$ that let us move from $s$ to $s'$, and $q(s,s';\mathcal{T})=0$ otherwise.
\end{definition}

While performing safe exploration, it is both desirable and useful to learn the transition function. While $\mathcal{T}$ implicitly performs this, it is useful to denote it as a function.

\begin{definition}
We therefore denote the transitions that have been learnt with certainity (without ambiguity) as function $f_c:\mathcal{S}\times\mathcal{A}\to\mathcal{S}\cup\{\gamma\}$:
\begin{align*}
	f_c(s,a;\mathcal{T}) \defeq \begin{cases}
	s' & \text{if } (s,a,s')\in\mathcal{T}\\
	\gamma & \text{otherwise}
	\end{cases}
\end{align*}
where $\gamma$ is a placeholder that represents ``an unknown state".
\end{definition}
Since the MDP is deterministic, $f_c$ is a properly defined function, i.e.
\begin{align*}
(s,a,s')\wedge(s,a,s'')\in\mathcal{T}_i \implies s' = s''
\end{align*}
We also overload the $f_c$ function as $\mathcal{S}\times\mathbb{A}\to\mathcal{S}\cup\{\gamma\}$ similar to $f$.

\begin{definition}
	We define the \emph{path-planning} function $g: \mathcal{S}\times\mathcal{S}\to \mathbb{A}$:
	\begin{align*}
	g(s,s';\mathcal{T}) \defeq\!\begin{cases}
	\argmin_{A: f_c(s,A;\mathcal{T})=s'} \abs{A} & \text{if } q(s,s';\mathcal{T})=1\\
	\emptyset & \text{otherwise}
	\end{cases}
	\end{align*}
	That is, $g(s,s';\mathcal{T})$ gives the smallest list of actions that moves from $s$ to $s'$ if such a transition is known in $\mathcal{T}$. In the case that there exist several such sequences, we assume $g(s,s';\mathcal{T})$ gives any one of them.
\end{definition}

\begin{remark}
	\label{fact:we_can_plan}
	Due to the assumption that we know how to move from one state to another inside the initial safe set, we have the following two relations for $\forall i$:
	\begin{align*}
	s,s' \in S_0 \implies q(s,s';\mathcal{T}_i)=1\\
	s,s' \in S_0, s\neq s' \implies g(s,s';\mathcal{T}_i) \neq \emptyset
	\end{align*}
\end{remark}

%
%

\subsection{Efficient Exploration}

In order to efficiently expand the safe set, we must optimize for the list of actions that will lead to the largest safe set expansion with minimum number of actions. We will do this in two steps:
\begin{enumerate}
  \item We first define a measure that corresponds to the possible safe set expansion from taking a new action.
  \item We then greedily optimize to take actions which are considered efficient under that measure.
\end{enumerate}

The most straightforward approach for defining a measure that corresponds to safe set expansion is to measure the amount of safe set expansion for each possible outcome of an action, and compute an expected value over all such possibilities. We call this the safe set expansion measure. 

However, as we will practically demonstrate in our results, this approach can be highly suboptimal especially for continuous dynamics. Instead, we develop a second measure that quantifies the uncertainty reduction on the outcomes of all state-action pairs by taking an action. This prioritizes exploration towards the safe set boundary in addition to exploring actions which will expand the boundary.

\begin{definition}
	We define the uncertainty reduction measure:
	\begin{align*}
	&\Delta(s,a;\mathcal{T}_i) \defeq\\
	&\sum_{s''\in U} p(s'';U)\sum_{\{s',a'\}} \abs{f_u(s',a';\mathcal{T}_i)\setminus f_u(s',a';\mathcal{T}_i\!\cup\!\{s,a,s''\})}
	\end{align*}
	where the inner summation is over all state-action pairs, $U=f_u(s,a;\mathcal{T}_i)$, and $p(s'';U)$ is the modeled probability that action $a$ will move from state $s$ to $s''$.
\end{definition}

While the underlying MDP is deterministic, for some actions we only know that they will end up in a specific set of points, thus it makes sense to model our uncertainty as a probability distribution over that set. While different probability models can be employed for $p(s'';U)$, we are going to use uniform distribution over $U$ for simplicity.

To show that our measure is well defined, we have the following theorem.

\begin{theorem}\label{th:nonincreasing_uncertainty}
The uncertainty does not increase when some more actions are taken. In fact, the updated uncertainty is a subset of the previous one. Formally, for all $s\in\mathcal{S}, a\in\mathcal{A}$, we have $\mathcal{T}_i\subseteq\mathcal{T}_j \implies f_u(s,a;\mathcal{T}_i) \supseteq f_u(s,a;\mathcal{T}_j)$ and thus $\Delta(s,a;\mathcal{T}_i) \geq 0$.
\end{theorem}
\begin{proof}
  $f_u(s,a;\mathcal{T}_j) = f_u(s,a;\mathcal{T}_i)\cap f_u(s,a;\mathcal{T}_j\setminus\mathcal{T}_i)\subseteq f_u(s,a;\mathcal{T}_i)$ for all $s$ and $a$.
\end{proof}

By Theorem~\ref{th:nonincreasing_uncertainty}, we can write:
\begin{align*}
&\Delta(s,a;\mathcal{T}_i) =\\
&\sum_{s''\in U}\!p(s'';U)\!\sum_{\{s',a'\}}\! \abs{f_u(s',a';\mathcal{T}_i)}\! -\!\abs{f_u(s',a';\mathcal{T}_i\!\cup\!\{s,a,s''\}\!)}
\end{align*}

Our goal is now to take steps to maximize the reduction in uncertainty. We realize that it is computationally difficult to maximize this reduction over sequences of multiple uncertain actions, so instead we will greedily optimize for the best immediate reduction in uncertainty. However, we still cannot directly maximize $\Delta(s,a;\mathcal{T}_i)$ over all $s$ and $a$, because at any step $i$ we are at a fixed state $s_i$ and getting to state $s$ may require many intermediate actions. Thus we must optimize over paths of actions $A \in \mathbb{A}$ starting at $s_i$. A second realization is that though it may be more efficient in some circumstances to try to get near a state $s$ via paths containing uncertain actions, it is difficult to optimize over such paths, thus we only consider paths to $s$ which consist of certain actions.

The following theorem will be useful for our algorithm:
\begin{theorem}\label{th:uncertainty_reduction}
	Taking actions that are already in the collection of knowledge, $\mathcal{T}_i$, does not lead to any safe set expansion.
\end{theorem}
\begin{proof}
	We first note that the update rule depends only on the uncertainty function $f_u$, which depends on $s$, $a$ and $\mathcal{T}_i$. Suppose $(s,a,s')\in\mathcal{T}_i$. When we take the action $a$ from state $s$, we do not add a new element to $\mathcal{T}_i$. Since $\mathcal{T}_i$ does not change, there will be no change in $f_u(s,a;\mathcal{T}_i)$ for any $s$ and $a$, so we will not be able to expand the safe set.
\end{proof}

Due to Theorem~\ref{th:uncertainty_reduction}, if we take a path of certain actions from $s_i$ to $s$ and then take an uncertain action at $s$, we know that the entire safe set expansion will come from the last uncertain step. Thus, we wish to perform the following optimization: 
\begin{align}
\max_{A\in\mathbb{A}, a\in\mathcal{A}} &\quad \frac{1}{\abs{A}+1}\Delta(s,a;\mathcal{T}_i) \nonumber\\
\text{subject to} &\quad s=f_c(s_i,A;\mathcal{T}_i),\nonumber\\
&\quad f_u(s,a;\mathcal{T}_i) \subseteq S_{i+1}
\end{align}

Again due to Theorem~\ref{th:uncertainty_reduction}, if two different action lists lead to the same state from state $s_i$, then the one with smaller cardinality is better off in the optimization. This means, we can just optimize over the shortest paths between state $s_i$ and the optimization variable state $s$. The optimization can then be reformulated as follows:
\begin{align}
\max_{s\in\mathcal{S}, a\in\mathcal{A}} &\quad \frac{1}{\abs{g(s_i,s;\mathcal{T}_i)}+1}\Delta(s,a;\mathcal{T}_i)\nonumber\\
\text{subject to} &\quad q(s_i,s;\mathcal{T}_i)=1,\nonumber\\
&\quad f_u(s,a;\mathcal{T}_i) \subseteq S_{i+1}
\label{eq:optimization}
\end{align}
This optimization is over finite variables for finite discrete MDPs.

\subsection{Overall Algorithm}
We first present the pseudocodes for the algorithm blocks that we so far formalized. Algorithm~\ref{alg:safe_set_expansion} is the pseudocode for computing $\bar{E}_{\mathcal{T}_i}(S_i)$, which corresponds to the procedure that we use for expanding the safe set at iteration $i$.

\begin{algorithm}[h]
	\caption{Safe Set Expansion}
	\label{alg:safe_set_expansion}
	\begin{algorithmic}[1]
		\Procedure{ExpandSafeSet}{$S_i,\mathcal{T}_i,f_u$}
		\State $S_{i+1}\gets S_i$
		\While {$S_{i+1}$ not converged}
		\State $S_{i+1}\!\gets\!S_{i+1}\cup\{s | f_u(s,a;\mathcal{T}_i)\!\subseteq\!S_{i+1} \text{ for } \exists a\!\in\!\mathcal{A}\}$
		\EndWhile
		\State \Return $S_{i+1}$
		\EndProcedure
	\end{algorithmic}
\end{algorithm}

In order to perform the optimization for efficient exploration, we compute expected total uncertainty reduction as in Algorithm~\ref{alg:expected_uncertainty_reduction}. Then, Algorithm~\ref{alg:greedy_optimization} benefits from that procedure to optimize for uncertainty reduction. Note that in Algorithm~\ref{alg:expected_uncertainty_reduction}, we use the scaling factor of $\frac{1}{\abs{U}}$ in line 7, since we are using a uniform distribution over the possible outcomes of an action.

\begin{algorithm}[h]
	\caption{Expected Uncertainty Reduction Computation}
	\label{alg:expected_uncertainty_reduction}
	\begin{algorithmic}[1]
		\Procedure{ExpectedReduction}{$f_u, s, a, L_s, L_a$}
		\State $U \gets f_u(s,a;\mathcal{T}_i)$
		\State $v \gets 0$
		\For {$s''\in U$}
    \State $f_u'(s',a';\mathcal{T}_i)\!\gets\!f_u(s',a';\mathcal{T}_i)\cap \phi(s'',L_sd^s(s,s')+L_ad^a(a,a'))$ ($\forall s'\in\mathcal{S}, \forall a'\in\mathcal{A}$)
    \State $\Delta(s',a')\!\gets\!\abs{f_u(s',a';\mathcal{T}_i)}\!-\!\abs{f_u'(s',a';\mathcal{T}_i)}$ $(\forall s'\in\mathcal{S}, \forall a'\in\mathcal{A})$
		\State $v \gets v + \frac{1}{\abs{U}}\sum_{s'\in\mathcal{S},a'\in\mathcal{A}} \Delta(s',a')$
		\EndFor
		\State \Return $v$
		\EndProcedure
	\end{algorithmic}
\end{algorithm}

\begin{algorithm}[h]
	\caption{Optimization for Uncertainty Reduction}
	\label{alg:greedy_optimization}
	\begin{algorithmic}[1]
		\Procedure{OptimizeGreedily}{$S_{i+1},\!s_i, \mathcal{T}_i,\!f_u, L_s,\!L_a$}
		\State $V \gets 0$
		\For {$s\in S_{i+1}$}
		\State $G \gets g(s_i,s;\mathcal{T}_i)$\Comment{Shortest-path algorithm}
		\If {$G = \emptyset$ and $s_i \neq s$}
		\State  \textbf{continue}
		\EndIf
		\For {$a\in\mathcal{A}$}
		\State $v\gets \textsc{ExpectedReduction}(f_u,s,a,L_s,L_a)$
		\State $v\gets v/(\abs{G}+1)$
		\If {$v > V$}
		\State $V\gets v$
		\State $G^*, a^* \gets G,a$
		\EndIf
		\EndFor
		\EndFor
		\State \Return $G^*, a^*$
		\EndProcedure
	\end{algorithmic}
\end{algorithm}

Lastly, we note that it is possible to have state-action pairs that are not in the current knowledge set, but have only one possible outcome. Adding these pairs to the knowledge in each iteration can possibly increase efficiency. We present this procedure in Algorithm~\ref{alg:expand_knowledge_collection}. Note we check if $\abs{U}=1$ for discrete MDPs. For continuous MDPs, where we can sample state and action spaces as we will describe in Section~\ref{sec:simulations_and_results}, we check if $U$ is a singleton.

\begin{algorithm}[h]
	\caption{Knowledge Expansion (Discrete MDPs)}
	\label{alg:expand_knowledge_collection}
	\begin{algorithmic}[1]
		\Procedure{ExpandKnowledge}{$\mathcal{T}_i,f_u$}
		\For {$s\in \mathcal{S}$}
		\For {$a\in \mathcal{A}$}
		\State $U\gets f_u(s,a;\mathcal{T}_i)$\Comment{$U_1$ is the first element}
		\If {$\abs{U}=1$}
			\State $\mathcal{T}_i \gets \mathcal{T}_i \cup \{(s,a,U_1)\}$
		\EndIf
		\EndFor
		\EndFor
		\State \Return $\mathcal{T}_i$
		\EndProcedure
	\end{algorithmic}
\end{algorithm}

We present the complete algorithm as a pseudocode in Algorithm~\ref{alg:pseudocode}.

\begin{algorithm}[h]
	\caption{Efficient and Safe Exploration (Discrete MDPs)}
	\label{alg:pseudocode}
	\begin{algorithmic}[1]
		\Require $n$\Comment{Iteration count}
		\Require $L_s, L_a$\Comment{Lipschitz-continuity constants}
		\Require $S_0\subseteq \mathcal{S}$ \Comment{Initial safe set}
		\Require $\mathcal{T}_0\subset \mathcal{S}\!\times\!\mathcal{A}\!\times\!\mathcal{S}$\Comment{Initial knowledge}
		\Require $s_0\in S_0$\Comment{Initial state}
		\State Compute $f_u(s,a;\mathcal{T}_0)$ for $\forall s\in\mathcal{S}$, $\forall a\in\mathcal{A}$
		\For {$i \gets 0$ to $n-1$}
		\State $\mathcal{T}_i \gets \textsc{ExpandKnowledge}(\mathcal{T}_i,f_u)$
		\State $S_{i+1} \gets \textsc{ExpandSafeSet}(S_i,\mathcal{T}_i,f_u)$
		\State $G^*,\!a^*\!\gets\!\textsc{OptimizeGreedily}(S_{i+1},\!s_i,\!\mathcal{T}_i,\!f_u,\!L_s,\!L_a)$
		\State $s^* \gets f(s_i,G^*)$\Comment{Take the certain actions}
		\State $s_{i+1} \gets f(s^*,a^*)$\Comment{Take the uncertain action}
		\State $\mathcal{T}_{i+1} \gets \mathcal{T}_i\cup (s^*,a^*,s_{i+1})$
    \State $f_u(s,a;\mathcal{T}_{i+1}\!)\!\gets\!f_u(s,a;\mathcal{T}_i)\cap \phi(s_{i+1},L_sd^s(s,s_{i+1}\!)+L_ad^a(a,a^*))$ $(\forall s\in\mathcal{S}, \forall a\in\mathcal{A})$
		\EndFor
	\end{algorithmic}
\end{algorithm}

\section{Simulations and Results}
\label{sec:simulations_and_results}

We call the algorithm developed above \textbf{Safe Exploration Optimized For Uncertainty Reduction}. We developed the following alternative methods as baselines to compare our algorithm against:

\noindent\textbf{Random Exploration.} In this method, we perform the safe set expansion as in our algorithm. However, each action taken is chosen randomly from the set of all possible actions, and is not necessarily safe.

\noindent\textbf{Safe Exploration with No Optimization.} In this method, we again perform the safe set expansion as in our algorithm. However, each action taken is chosen randomly from the set of actions that are classified as safe at the current state.

\noindent\textbf{Safe Exploration Optimized for Safe Set Expansion.} This is similar to Safe Exploration Optimized for Uncertainty Reduction. However, instead of optimizing for the uncertainty reduction measure, we optimize for the safe set expansion measure described earlier. If the maximum expected safe set expansion amount is zero, we take the safe action that will, expectedly, push the system to its closest safety boundary at that time, so that it can possibly expand the safe set later. 

We simulated two different environments with continuous state and action spaces, described below, to analyze the performance of our algorithm. In each environment, we used Euclidean distance for both $d^s$ and $d^a$. We used \emph{breadth-first search (BFS)} for the $g$ function.

For both environments, we began by uniformly sampling the state and action spaces. We used only those original samples when calculating the $\Delta$-function, and not any new states we might have encountered since starting the simulation, so that the optimization would not be biased towards already visited states.

To quantitatively assess the performance of our algorithm in comparison with the baselines, we defined and used the following metrics:
\begin{itemize}
	\item \textbf{Safe Set Size}. We plot the number of actions vs. $\abs{S_i}$ to evaluate the safe set expansion efficiency.
	\item \textbf{Total Uncertainty.} We plot the number of actions vs. $\sum_{s\in\mathcal{S},a\in\mathcal{A}}\abs{f_u(s,a;\mathcal{T}_i)}$ to analyze how fast the total uncertainty decreases. For consistency among the iterations, we sum only over states in the original sampling, i.e. we do not consider other states encountered since starting the simulation. 
\end{itemize}

\subsection{Muddy Jumper}
We simulated a simple system with the transition model:
\begin{align*}
f(s,a) = s + a(1-\psi(s))
\end{align*}
where $a\in\mathcal{A}=[-C, C]$, $s\in\mathcal{S}=\mathbb{R}$, and $\psi(s)$ is the dampening factor. We simulated it as:
\begin{align*}
\psi(s) = \begin{cases}
0 & \text{if } \abs{s}<A\\
\frac{\abs{s}-A}{B-A} & \text{if } A\leq\abs{s}<B\\
1 & \text{otherwise}
\end{cases}
\end{align*}
which is plotted in Fig.~\ref{fig:dampening_factor}. This environment was inspired by the idea of a robot jumping on muddy ground. When the dampening factor is $1$, the robot is not able to move anymore, so those states are unsafe.

\begin{figure}[h]
	\centering
	\includegraphics[width=\columnwidth]{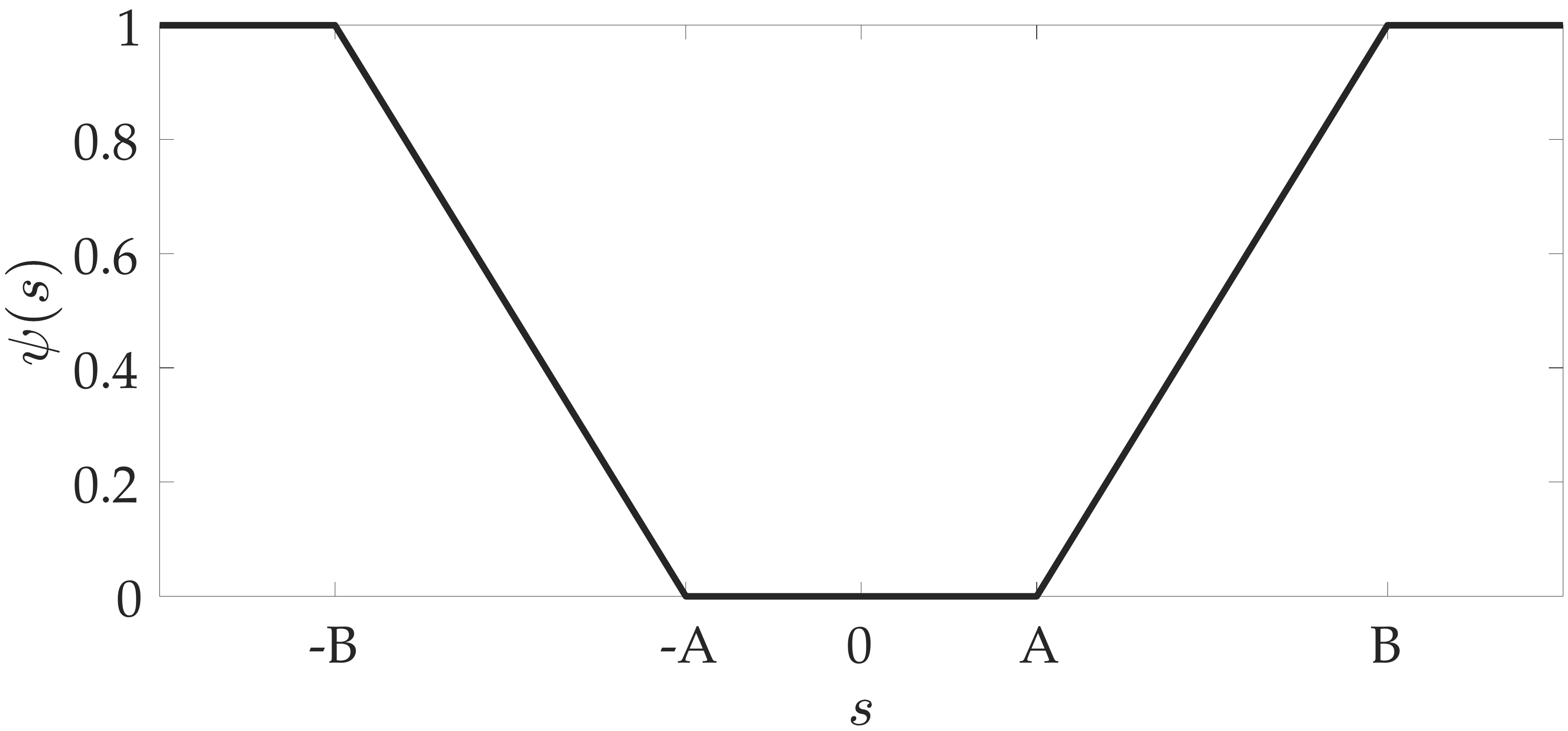}
	\vspace*{-22pt}
	\caption{Dampening factor for the Muddy Jumper environment. The regions that are inside $[-A,A]$ are not muddy, so the robot can move without any dampening. The outer regions have increasing amount of mud, which causes the robot's movements to be dampened. The robot becomes completely immobile outside $(-B,B)$.}
	\label{fig:dampening_factor}
	\vspace*{-10pt}
\end{figure}

With these, we can take $L_a=1$ and $L_s=\frac{C+B-A}{B-A}$. For our simulations, we used $A=3$, $B=9$ and $C=12$. We sample $\mathcal{S}$ as $\{-10, -9.8, \dots, 9.8, 10\}$ and $\mathcal{A}$ as $\{-12, -11.8, \dots, 11.8, 12\}$. Hence, the largest safe set is the interval between $-8.8$ and $8.8$. We set $S_0=[-3,3]$ and $s_0=0$. We present the results of our algorithm in Fig.~\ref{fig:muddy_jumper_results}.

\begin{figure}[h]
	\centering
	\includegraphics[width=\columnwidth]{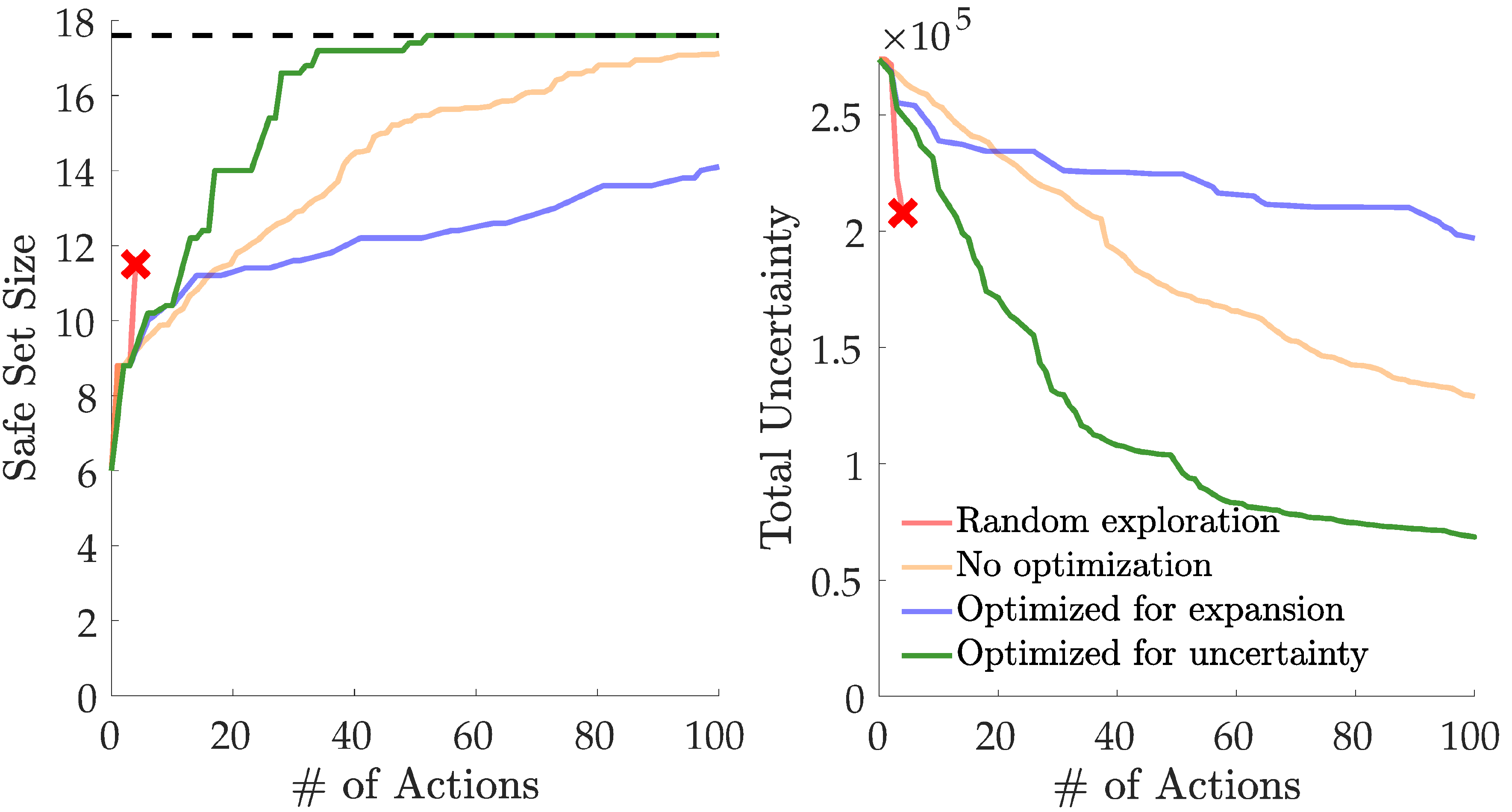}
	\vspace*{-16pt}
	\caption{Results of Muddy Jumper environment with our algorithm and the three other baselines. For safe exploration with no optimization, 10 runs have been averaged. The crosses at the end of random exploration lines indicate the system reached to an unsafe state. \textbf{Left.} Number of actions vs. safe set size. Dashed line shows the size of largest possible safe set. \textbf{Right.} Number of actions vs. total uncertainty.}
	\label{fig:muddy_jumper_results}
	\vspace*{-10pt}
\end{figure}

It can be seen that random exploration leads to fast uncertainty reduction, because the allowed actions include very large jumps. However more importantly, it terminates upon reaching unsafe states after only a few actions. Other exploration techniques tackle this problem by taking only safe actions. However, our algorithm outperforms the baselines in terms of efficiency. It is better than the safe exploration with no optimization as it optimizes the uncertainty reduction per action. It can be seen that optimizing for safe set expansion yields good results initially; however becomes highly suboptimal afterwards. In this case, the superiority of our algorithm can be due to that reducing overall uncertainty leads to larger safe set expansion in future iterations, whereas a greedy optimization for the expansion cannot achieve this. As a side note, it is also interesting to see that the safe exploration with no optimization outperformed the safe exploration optimized for safe set expansion in later iterations. This might be because optimizing for expansion may often get stuck at states that are not amongst the original state samples, so the optimization becomes only over the immediate next action.

\subsection{Hilly Jumper}
We simulate another environment with the following transition model:
\begin{align*}
f(s,a) = s + a - h'(s)
\end{align*}
where $a\in\mathcal{A}=[-C, C]$, $s\in\mathcal{S}=\mathbb{R}$, and $h(s)$ is an environment-dependent function. We simulated it as:
\begin{align*}
h(s) = \begin{cases}
-\frac{(s+A)^4}{4B^4}+\frac{(s+A)^2}{2B^2} & \text{if } s<-A\\
0 & \text{if } -\!A\leq s<A\\
-\frac{(s-A)^4}{4B^4}+\frac{(s-A)^2}{2B^2} & \text{otherwise}\\
\end{cases}
\end{align*}
Then, we have
\begin{align*}
h'(s) = \begin{cases}
-\frac{(s+A)^3}{B^4} + \frac{s + A}{B^2} & \text{if } s< -A\\
0 & \text{if } -\!A\leq s<A\\
-\frac{(s-A)^3}{B^4} + \frac{s - A}{B^2} & \text{otherwise}\\
\end{cases}
\end{align*}
This environment was inspired by a robot jumping on hilly ground where $h(s)$ is the elevation function. Both the elevation function and its derivative are plotted in Fig.~\ref{fig:elevation}. When $\abs{s} > A+B$ and $C < \abs{h'(s)}$, the robot is on a very steep terrain, so it cannot return to the safe set, which is around $0$-state.

\begin{figure}[h]
	\centering
	\includegraphics[width=\columnwidth]{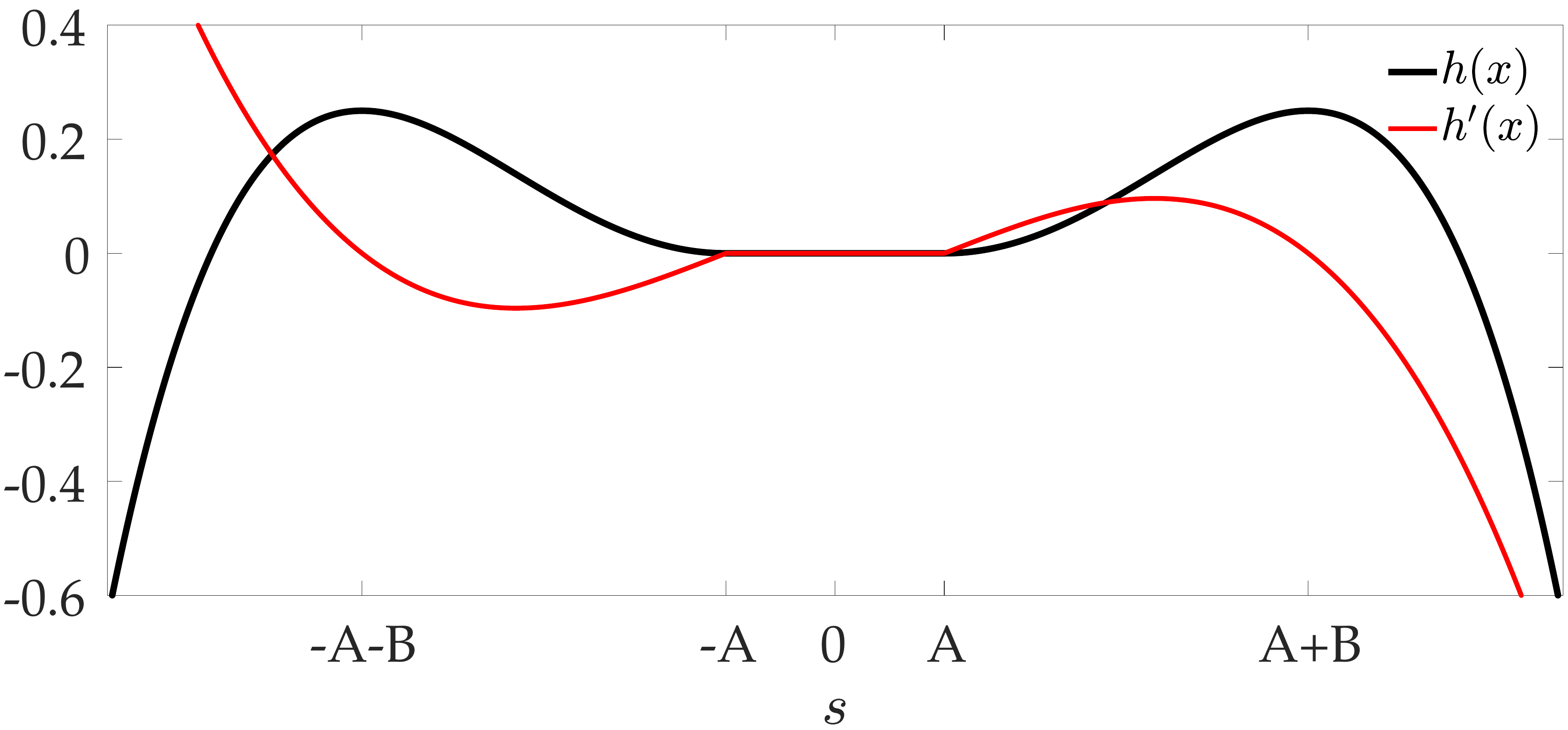}
	\vspace*{-22pt}
	\caption{Elevation function and its derivative are plotted for the Hilly Jumper environment. The robot has a constant maximum power, which prevents it from returning to the central region when it reaches points which are too steep.}
	\label{fig:elevation}
\end{figure}

With these, we can take $L_a=1$. While this problem does not have a single Lipschitz-continuity constant over all states, as the slope increases without bounds in each direction, it locally satisfies Lipschitz-continuity around the central region. For our simulations, we used $A=1.2$, $B=4$, $C=0.3$ and $L_s=1.4$. We sample $\mathcal{S}$ as $\{-6.9,-6.8,\dots,6.8,6.9\}$ and $\mathcal{A}$ as $\{-0.3,-0.2,\dots,0.2,0.3\}$. Hence, the largest safe set is the interval between $-6.6$ and $6.6$. We set $S_0=[-1.2,1.2]$ and $s_0=0$. We present the results of our algorithm in Fig.~\ref{fig:hilly_jumper_results}.

\begin{figure}[h]
	\centering
	\includegraphics[width=\columnwidth]{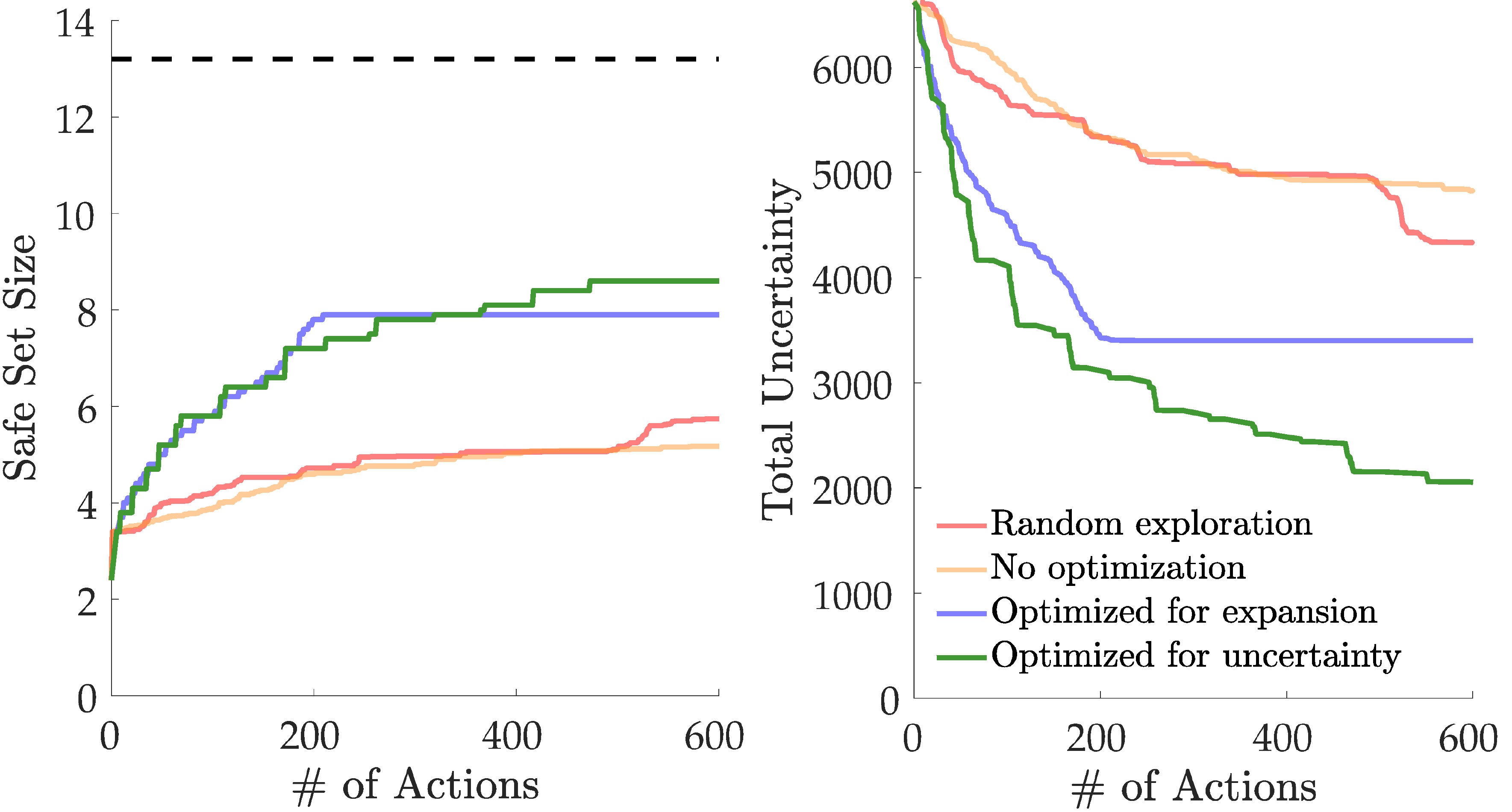}
	\vspace*{-12pt}
	\caption{Results of Hilly Jumper environment with our algorithm and the three other baselines. For both random exploration and safe exploration with no optimization methods, 10 runs have been averaged. \textbf{Left.} Number of actions vs. safe set size. Dashed line shows the size of largest possible safe set. \textbf{Right.} Number of actions vs. total uncertainty.}
	\label{fig:hilly_jumper_results}
\end{figure}

Unlike the Muddy Jumper environment, random exploration does not lead to very fast uncertainty reduction, because the action set contains only very short jumps. Due to the same reason, it does not crash \textemdash it is unlikely for the robot to leave the safe set with random small jumps. It is even harder due to the fact that the regions $A < \abs{s} < A+B$ have a slope towards the central region. Safe exploration with no optimization performs even worse than random exploration, because it enforces safety constraint.

Similar to Muddy Jumper experiments, the baseline which optimizes over safe set expansion initially gives good performance; and later becomes suboptimal. In this case, the reason is the following: After the algorithm expands the safe set up to $-6.2$ on the negative side, it cannot expand further due to the sampling of state space. It also could not expand the positive side further from $1.7$, because the algorithm gets stuck near the negative limit. This is because when the system is at a new state near $-6.2$, all actions have some level of uncertainty, so the optimization is only over the immediate actions. And when the system is near the limit, it does not leave that region; because if it finds some possibility of safe set expansion, it explores; if it cannot find a possible expansion, it still goes toward the boundary. In fact, we ran this algorithm up to $2000$ actions, and observed that the system was always in states $[6.10,6.20]$ after $220^{\text{th}}$ action. This baseline would require the following additional mechanisms to perform well: 1) to spot when to be confident that there is no possibility of expansion, and 2) to make sure the system moves to unexplored state regions when that confidence is obtained. 

On the other hand, our algorithm, safe exploration optimized for uncertainty reduction, outperforms all the baseline methods in terms of efficiency. It can be noted that none of the algorithms reaches the maximum expandable safe set within $600$ actions. While the explored safe set might be expanded further with more iterations, it is also limited due to the sampling of state space \textemdash denser sampling could increase this limit; however it causes a significant computational burden.

\section{Discussion}
\label{sec:discussion}
For safe exploration tasks, computational cost is a problem in general. Our algorithm has polynomial complexity in the number of states and actions. While the use of Gaussian processes enables faster computation, directly using Lipschitz-continuity makes the algorithm computationally heavier, though we do note that our algorithm is parallelizable. For reference, we initially sample 101 states, 121 possible actions in Muddy Jumper; and 139 states, 7 possible actions in Hilly Jumper. However, the number of states increases during algorithm execution as the system visits states that are not amongst the initial samples. Additionally, it can be a concern for low-memory systems that our framework requires the storage of $\mathcal{T}_i$ that grows linearly with the number of uncertain actions taken.

Both our example environments had 1-dimensional state spaces. We note that in higher dimensional problems, the number of states necessarily grows exponentially in the dimension of state space. In particular, the number of states on the boundary of the safe set at any step is likewise exponential in the dimension of the state space. Since our algorithm does not extrapolate from data in order to produce its safety guarantees, it by design must explore this exponentially sized safe set boundary.

For some specific applications, our algorithm's requirement that how to move from one state to another is known inside the initial safe set can be too restrictive. In such cases, our algorithm can be readily applied provided that there exist some uncertain but safe actions for each state in the initial safe set. While this may hurt efficiency, it will enable the use of our algorithm in broader configurations.

In this formulation, our algorithm is limited to deterministic environments. Further research could generalize it to stochastic MDPs and to systems with disturbances. Similarly, our framework requires prior knowledge of the Lipschitz continuity parameters $L_s$ and $L_a$. In settings where it is impractical to provide estimates of these parameters prior to running this algorithm, this algorithm could be modified to learn them online. However, either of these generalization would come at the expense of losing the algorithm's deterministic guarantees.

As long as Lipschitz-continuity assumptions can be made, our algorithm can be applied to both linear and nonlinear systems, as well as to systems where safe state set boundaries are very complex. We have demonstrated our algorithm on two simulated environments, and we are planning to design real robotics experiments to showcase our algorithm.

Lastly, in each iteration of our algorithm, we currently take only actions we are certain about before taking a final, uncertain, action that we learn from (see the first constraint in \eqref{eq:optimization}). This algorithm could potentially be improved by having it optimize over and learn from paths that include several uncertain actions in sequence rather than just one.

\section{Conclusion}
\label{sec:conclusion}
In this paper, we presented an algorithm to safely explore safety-critical deterministic MDPs that is efficient in terms of the number of actions it takes. Unlike some previous works, our algorithm does not require the transition function to be known a priori, except for some little prior knowledge. 

Future work will demonstrate our algorithm's use in practice. In addition, future work can be done to further improve the efficiency of our algorithm by allowing it to plan along sequences of multiple uncertain actions. We are also planning to relax the determinism requirement on the MDP and apply our algorithm to stochastic environments. And lastly, exploration is needed into combining this Lipschitz grounded approach with model-based approaches to handle higher dimensional state and action spaces.  

Finally, we will study different methods for transferring from a source (e.g., simulation) domain to a target (e.g., real-world) domain. In order for a robotic system to adapt to a new domain, the system must often explore the parameters of the new environment, but must also do so safely.  In the future work, we will leverage our work on safe exploration in MDPs and the Delaunay-based optimization \cite{alimo2017optimization} to address this problem. 

\section*{Acknowledgment}

The authors thank Fred Y. Hadaegh, Adrian Stoica and Duligur Ibeling for the discussions and support. The authors also gratefully acknowledge funding from Jet Propulsion Laboratory, California Institute of Technology, under a contract with the National Aeronautics and Space Administration in support of this work. Toyota Research Institute (``TRI") provided funds to assist the authors with their research but this article solely reflects the opinions and conclusions of its authors and not TRI or any other Toyota entity.


\bibliographystyle{ieeetran}
\bibliography{refs}

\begin{thebibliography}{10}
\providecommand{\url}[1]{#1}
\csname url@rmstyle\endcsname
\providecommand{\newblock}{\relax}
\providecommand{\bibinfo}[2]{#2}
\providecommand\BIBentrySTDinterwordspacing{\spaceskip=0pt\relax}
\providecommand\BIBentryALTinterwordstretchfactor{4}
\providecommand\BIBentryALTinterwordspacing{\spaceskip=\fontdimen2\font plus
\BIBentryALTinterwordstretchfactor\fontdimen3\font minus
  \fontdimen4\font\relax}
\providecommand\BIBforeignlanguage[2]{{%
\expandafter\ifx\csname l@#1\endcsname\relax
\typeout{** WARNING: IEEEtran.bst: No hyphenation pattern has been}%
\typeout{** loaded for the language `#1'. Using the pattern for}%
\typeout{** the default language instead.}%
\else
\language=\csname l@#1\endcsname
\fi
#2}}

\bibitem{sadigh2016safe}
D.~Sadigh and A.~Kapoor, ``Safe control under uncertainty with probabilistic
  signal temporal logic,'' in \emph{Proceedings of Robotics: Science and
  Systems ({RSS})}, June 2016.

\bibitem{dey2016fast}
D.~Dey, D.~Sadigh, and A.~Kapoor, ``Fast safe mission plans for autonomous
  vehicles,'' Proceedings of Robotics: Science and Systems Workshop, Tech.
  Rep., June 2016.

\bibitem{katz2017reluplex}
G.~Katz, C.~Barrett, D.~L. Dill, K.~Julian, and M.~J. Kochenderfer, ``Reluplex:
  An efficient smt solver for verifying deep neural networks,'' in
  \emph{International Conference on Computer Aided Verification}.\hskip 1em
  plus 0.5em minus 0.4em\relax Springer, 2017, pp. 97--117.

\bibitem{argall2009survey}
B.~D. Argall, S.~Chernova, M.~Veloso, and B.~Browning, ``A survey of robot
  learning from demonstration,'' \emph{Robotics and autonomous systems},
  vol.~57, no.~5, pp. 469--483, 2009.

\bibitem{kober2012reinforcement}
J.~Kober and J.~Peters, ``Reinforcement learning in robotics: A survey,'' in
  \emph{Reinforcement Learning}.\hskip 1em plus 0.5em minus 0.4em\relax
  Springer, 2012, pp. 579--610.

\bibitem{coraluppi1999risk}
S.~P. Coraluppi and S.~I. Marcus, ``Risk-sensitive and minimax control of
  discrete-time, finite-state markov decision processes,'' \emph{Automatica},
  vol.~35, no.~2, pp. 301--309, 1999.

\bibitem{heger1994consideration}
M.~Heger, ``Consideration of risk in reinforcement learning,'' in \emph{Machine
  Learning Proceedings 1994}.\hskip 1em plus 0.5em minus 0.4em\relax Elsevier,
  1994, pp. 105--111.

\bibitem{sato2001td}
M.~Sato, H.~Kimura, and S.~Kobayashi, ``Td algorithm for the variance of return
  and mean-variance reinforcement learning,'' \emph{Transactions of the
  Japanese Society for Artificial Intelligence}, vol.~16, no.~3, pp. 353--362,
  2001.

\bibitem{borkar2002q}
V.~S. Borkar, ``Q-learning for risk-sensitive control,'' \emph{Mathematics of
  operations research}, vol.~27, no.~2, pp. 294--311, 2002.

\bibitem{gaskett2003reinforcement}
C.~Gaskett, ``Reinforcement learning under circumstances beyond its control,''
  \emph{Proceedings of the International Conference on Computational
  Intelligence for Modelling Control and Automation (CIMCA2003)}, February
  2003.

\bibitem{geibel2005risk}
P.~Geibel and F.~Wysotzki, ``Risk-sensitive reinforcement learning applied to
  control under constraints,'' \emph{Journal of Artificial Intelligence
  Research}, vol.~24, pp. 81--108, 2005.

\bibitem{aswani2013provably}
A.~Aswani, H.~Gonzalez, S.~S. Sastry, and C.~Tomlin, ``Provably safe and robust
  learning-based model predictive control,'' \emph{Automatica}, vol.~49, no.~5,
  pp. 1216--1226, 2013.

\bibitem{abbeel2005exploration}
P.~Abbeel and A.~Y. Ng, ``Exploration and apprenticeship learning in
  reinforcement learning,'' in \emph{Proceedings of the 22nd international
  conference on Machine learning}.\hskip 1em plus 0.5em minus 0.4em\relax ACM,
  2005, pp. 1--8.

\bibitem{abbeel2010autonomous}
P.~Abbeel, A.~Coates, and A.~Y. Ng, ``Autonomous helicopter aerobatics through
  apprenticeship learning,'' \emph{The International Journal of Robotics
  Research}, vol.~29, no.~13, pp. 1608--1639, 2010.

\bibitem{berkenkamp2016bayesian}
F.~Berkenkamp, A.~Krause, and A.~P. Schoellig, ``Bayesian optimization with
  safety constraints: safe and automatic parameter tuning in robotics,''
  \emph{arXiv preprint arXiv:1602.04450}, 2016.

\bibitem{ames2017control}
A.~D. Ames, X.~Xu, J.~W. Grizzle, and P.~Tabuada, ``Control barrier function
  based quadratic programs for safety critical systems,'' \emph{IEEE
  Transactions on Automatic Control}, vol.~62, no.~8, pp. 3861--3876, 2017.

\bibitem{akametalu2014reachability}
A.~K. Akametalu, S.~Kaynama, J.~F. Fisac, M.~N. Zeilinger, J.~H. Gillula, and
  C.~J. Tomlin, ``Reachability-based safe learning with gaussian processes.''
  in \emph{53rd IEEE Conference on Decision and Control (CDC)}.\hskip 1em plus
  0.5em minus 0.4em\relax Citeseer, 2014, pp. 1424--1431.

\bibitem{garcia2015comprehensive}
J.~Garc{\i}a and F.~Fern{\'a}ndez, ``A comprehensive survey on safe
  reinforcement learning,'' \emph{Journal of Machine Learning Research},
  vol.~16, no.~1, pp. 1437--1480, 2015.

\bibitem{turchetta2016safe}
M.~Turchetta, F.~Berkenkamp, and A.~Krause, ``Safe exploration in finite markov
  decision processes with gaussian processes,'' in \emph{Advances in Neural
  Information Processing Systems}, 2016, pp. 4312--4320.

\bibitem{wachi2018safe}
A.~Wachi, Y.~Sui, Y.~Yue, and M.~Ono, ``Safe exploration and optimization of
  constrained mdps using gaussian processes,'' in \emph{AAAI Conference on
  Artificial Intelligence (AAAI)}, 2018.

\bibitem{gillula2011guaranteed}
J.~H. Gillula and C.~J. Tomlin, ``Guaranteed safe online learning of a bounded
  system,'' in \emph{Intelligent Robots and Systems (IROS), 2011 IEEE/RSJ
  International Conference on}.\hskip 1em plus 0.5em minus 0.4em\relax IEEE,
  2011, pp. 2979--2984.

\bibitem{garcia2012safe}
J.~Garcia and F.~Fern{\'a}ndez, ``Safe exploration of state and action spaces
  in reinforcement learning,'' \emph{Journal of Artificial Intelligence
  Research}, vol.~45, pp. 515--564, 2012.

\bibitem{moldovan2012safe}
T.~M. Moldovan and P.~Abbeel, ``Safe exploration in markov decision
  processes,'' \emph{arXiv preprint arXiv:1205.4810}, 2012.

\bibitem{sutton1998reinforcement}
R.~S. Sutton, A.~G. Barto, \emph{et~al.}, \emph{Reinforcement learning: An
  introduction}.\hskip 1em plus 0.5em minus 0.4em\relax MIT press, 1998.

\bibitem{alimo2017optimization}
S.~R. Alimo, P.~Beyhaghi, and T.~R. Bewley, ``Optimization combining
  derivative-free global exploration with derivative-based local refinement,''
  in \emph{2017 IEEE 56th Annual Conference on Decision and Control
  (CDC)}.\hskip 1em plus 0.5em minus 0.4em\relax IEEE, 2017, pp. 2531--2538.

\end{thebibliography}

\end{document}